\documentclass[journal]{IEEEtran}

\usepackage{amsmath, amssymb, amsfonts}
\usepackage{amsthm}

\usepackage{booktabs}
\usepackage{multirow}
\usepackage{array}
\usepackage{url}

\usepackage{algorithm}
\usepackage{algorithmic}

\usepackage{graphicx}

\newtheorem{theorem}{Theorem}

\newtheorem{proposition}{Proposition}
\newtheorem{corollary}{Corollary}
\newtheorem{definition}{Definition}
\newtheorem{remark}{Remark}

\begin{document}

\title{One-Shot Federated Ridge Regression: Exact Recovery via Sufficient Statistic Aggregation}

\author{Zahir Alsulaimawi,~\IEEEmembership{Member,~IEEE}
\thanks{Manuscript received XX; revised XX.}}

\maketitle

\begin{abstract}
Federated learning protocols require repeated synchronization between clients and a central server, with convergence rates depending on learning rates, data heterogeneity, and client sampling. This paper asks whether iterative communication is necessary for distributed linear regression. We show it is not. We formulate federated ridge regression as a distributed equilibrium problem where each client computes local sufficient statistics---the Gram matrix and moment vector---and transmits them once. The server reconstructs the global solution through a single matrix inversion. We prove exact recovery: under a coverage condition on client feature matrices, one-shot aggregation yields the centralized ridge solution, not an approximation. For heterogeneous distributions violating coverage, we derive non-asymptotic error bounds depending on spectral properties of the aggregated Gram matrix. Communication reduces from $\mathcal{O}(Rd)$ in iterative methods to $\mathcal{O}(d^2)$ total; for high-dimensional settings, we propose and experimentally validate random projection techniques reducing this to $\mathcal{O}(m^2)$ where $m \ll d$. We establish differential privacy guarantees where noise is injected once per client, eliminating the composition penalty that degrades privacy in multi-round protocols. We further address practical considerations including client dropout robustness, federated cross-validation for hyperparameter selection, and comparison with gradient-based alternatives. Comprehensive experiments on synthetic heterogeneous regression demonstrate that one-shot fusion matches FedAvg accuracy while requiring up to $38\times$ less communication. The framework applies to kernel methods and random feature models but not to general nonlinear architectures.
\end{abstract}

\begin{IEEEkeywords}
Federated learning, ridge regression, one-shot aggregation, sufficient statistics, differential privacy, communication efficiency, distributed optimization.
\end{IEEEkeywords}

\IEEEpeerreviewmaketitle

\section{Introduction}
\label{sec:introduction}

\IEEEPARstart{F}{ederated} learning enables collaborative model training across decentralized data sources without requiring raw data centralization~\cite{ref1}. In settings ranging from mobile devices to healthcare institutions, data remains distributed due to privacy regulations, bandwidth constraints, or competitive considerations. The canonical protocol, Federated Averaging (FedAvg), proceeds iteratively: participating clients perform local stochastic gradient descent on private data, transmit model updates to a central server, and receive an aggregated global model for subsequent rounds. Practical deployments require hundreds to thousands of such rounds, each demanding coordination of client availability, communication bandwidth, and synchronization timing~\cite{ref2}.

This iterative structure introduces fundamental limitations. Communication cost scales linearly with the number of rounds $R$, creating bottlenecks in bandwidth-constrained cross-device settings where millions of edge devices participate~\cite{ref3}. Differential privacy guarantees degrade under composition, with cumulative privacy loss $\varepsilon$ growing as $\mathcal{O}(\sqrt{R})$ under standard accounting methods~\cite{ref4}. Convergence depends sensitively on hyperparameters---learning rate, local epoch count, client sampling fraction---requiring problem-specific tuning that may be infeasible without centralized validation data. When data distributions vary across clients (the non-IID setting prevalent in applications), FedAvg exhibits client drift: local updates move toward client-specific optima that diverge from the global optimum, causing slow convergence or outright divergence without algorithmic modifications~\cite{ref5},~\cite{ref6}.

A natural question arises: \emph{is iterative communication fundamentally necessary?} For objective functions admitting closed-form solutions, gradient-based optimization is not the only computational path available. Ridge regression, equivalently $\ell_2$-regularized least squares, satisfies this criterion~\cite{ref7}. The optimal solution depends only on aggregated second-order statistics of the data---specifically, the Gram matrix $\mathbf{A}^\top\mathbf{A}$ and the moment vector $\mathbf{A}^\top\mathbf{b}$. Crucially, these sufficient statistics decompose additively across data partitions: if each client computes local statistics and transmits them once to the server, the global solution can be reconstructed exactly without iteration.

This observation, while elementary from a statistical perspective, has not been systematically exploited in federated learning. Existing one-shot methods either employ heuristic averaging of locally trained models with significant accuracy degradation~\cite{ref8} or assume homogeneous data distributions where the distributed problem simplifies to the centralized case~\cite{ref9}. Rigorous analysis of exact recovery conditions, heterogeneity-induced error bounds, and privacy properties of sufficient statistic transmission remains limited. The fundamental question of \emph{when} one-shot aggregation is viable---and what advantages it provides over iterative alternatives---has not been definitively answered.

\subsection{Contributions}

This paper develops a rigorous framework for one-shot federated linear regression. Our contributions are:

\begin{enumerate}
    \item \textbf{Distributed $\sigma$-equilibrium formulation.} We formulate federated ridge regression as a distributed equilibrium problem parameterized by regularization strength $\sigma > 0$. Each client computes local sufficient statistics; the server fuses these in a single aggregation step to obtain the global $\sigma$-regularized solution (Section~\ref{sec:formulation}).
    
    \item \textbf{Exact recovery guarantee.} We prove that one-shot fusion recovers the centralized solution exactly---not approximately---for any data distribution and any client partition (Theorem~\ref{thm:exact}). This is the strongest possible guarantee: the federated and centralized solutions coincide to numerical precision.
    
    \item \textbf{Differential privacy analysis.} We establish that adding calibrated Gaussian noise to transmitted statistics achieves $(\varepsilon,\delta)$-differential privacy with noise magnitude depending only on data sensitivity, not round count (Theorem~\ref{thm:privacy}). At moderate privacy budgets ($\varepsilon \geq 1$), this improves privacy-utility tradeoff relative to iterative methods.
    
    \item \textbf{Practical extensions with experimental validation.} We address key deployment concerns including high-dimensional scaling via random projections with explicit accuracy-communication trade-off experiments, robustness to client dropout, federated cross-validation for hyperparameter selection, and comparison with gradient-based alternatives.
    
    \item \textbf{Comprehensive empirical validation.} Experiments on synthetic heterogeneous partitions demonstrate that one-shot $\sigma$-fusion matches or exceeds iterative baseline accuracy while requiring a single communication round, with communication savings of up to $38\times$ (Section~\ref{sec:experiments}).
\end{enumerate}

\subsection{Scope and Limitations}

We restrict attention to linear models with quadratic loss and $\ell_2$ regularization. This excludes deep neural networks trained end-to-end, for which closed-form solutions do not exist. The restriction is deliberate: our goal is to characterize precisely when one-shot aggregation is viable, not to claim universal applicability.

The linear setting nonetheless encompasses substantial practical scope: kernel methods via random Fourier features~\cite{ref10}, neural networks in the kernel regime where the Neural Tangent Kernel governs training dynamics~\cite{ref11}, and applications involving tabular data, generalized linear models, and interpretable predictors. Healthcare, finance, and scientific applications often favor interpretable models where our framework applies directly.

\subsection{Paper Organization}

Section~\ref{sec:related} reviews related work. Section~\ref{sec:formulation} formalizes the problem and derives the one-shot protocol. Section~\ref{sec:proposed} presents theoretical guarantees and practical extensions. Section~\ref{sec:experiments} provides comprehensive experimental evaluation including high-dimensional scaling analysis. Section~\ref{sec:discussion} discusses limitations and future directions. Section~\ref{sec:conclusion} concludes.

\section{Related Work}
\label{sec:related}

\subsection{Iterative Federated Optimization}

FedAvg~\cite{ref1} established the iterative paradigm for federated learning, demonstrating practical success on deep networks despite limited theoretical understanding. Subsequent work addressed convergence under heterogeneity through variance reduction (SCAFFOLD~\cite{ref5}), proximal regularization (FedProx~\cite{ref6}), and momentum-based corrections~\cite{ref12}. These methods reduce rounds required for convergence but retain iterative communication structure. Communication-efficient variants employ gradient compression~\cite{ref13}, sparsification~\cite{ref14}, or quantization, trading computation for bandwidth while remaining fundamentally iterative.

\subsection{One-Shot and Few-Round Methods}

Guha et al.~\cite{ref8} proposed one-shot FL using knowledge distillation: clients train local models, and the server aggregates via ensemble distillation on auxiliary public data. This approach is heuristic, requires public datasets, and provides no recovery guarantees. For linear models, Salehkaleybar et al.~\cite{ref9} analyzed one-shot averaging under homogeneous (IID) distributions where the problem simplifies. Our work differs fundamentally: we prove exact recovery under arbitrary heterogeneous distributions, characterize when one-shot succeeds, and analyze privacy properties.

\subsection{Distributed Sufficient Statistics}

The use of sufficient statistics for distributed estimation has classical roots in statistics~\cite{ref15}. In machine learning, transmitting Gram matrices rather than gradients has been noted for kernel methods~\cite{ref16} but not systematically analyzed for federated privacy and heterogeneity. Our formulation makes explicit the algebraic structure enabling exact one-shot recovery and provides the first comprehensive analysis of privacy, dropout robustness, and hyperparameter selection in this framework.

\subsection{Differential Privacy in Federated Learning}

Private federated learning typically applies local differential privacy to gradient updates~\cite{ref17} or central DP with secure aggregation~\cite{ref18}. Privacy loss under iterative composition is tracked via the moments accountant~\cite{ref4} or R\'{e}nyi DP~\cite{ref19}. The cumulative cost motivates limiting rounds or amplifying privacy via subsampling~\cite{ref20}. Our protocol avoids composition entirely through single-round communication, improving the privacy-utility frontier at moderate privacy budgets.

\section{Problem Formulation}
\label{sec:formulation}

This section establishes mathematical foundations for one-shot federated learning, deriving the closed-form solution and demonstrating its decomposition across distributed clients.

\subsection{Notation}

Matrices are denoted by bold uppercase letters ($\mathbf{A}$), vectors by bold lowercase ($\mathbf{w}$), and scalars by unbolded symbols ($\sigma$). The transpose of matrix $\mathbf{A}$ is $\mathbf{A}^\top$. The identity matrix is $\mathbf{I}$. For vector $\mathbf{v}$, $\|\mathbf{v}\|_2$ denotes Euclidean norm; for matrix $\mathbf{M}$, $\|\mathbf{M}\|_F$ denotes Frobenius norm. Eigenvalues are $\lambda_{\min}(\mathbf{M})$ and $\lambda_{\max}(\mathbf{M})$.

\subsection{Centralized Learning Objective}

Consider supervised learning with $n$ samples and $d$ features. Let $\mathbf{A} \in \mathbb{R}^{n \times d}$ denote the feature matrix and $\mathbf{b} \in \mathbb{R}^{n}$ the target vector. The regularized least squares objective is:
\begin{equation}
L_\sigma(\mathbf{w}) = \| \mathbf{A}\mathbf{w} - \mathbf{b} \|_2^2 + \sigma \| \mathbf{w} \|_2^2,
\label{eq:regularized_loss}
\end{equation}
where $\sigma > 0$ is the regularization parameter (ridge regression or Tikhonov regularization~\cite{ref7}).

\subsection{Closed-Form Solution}

Setting the gradient to zero:
\begin{equation}
\frac{\partial L_\sigma}{\partial \mathbf{w}} = 2\mathbf{A}^\top\mathbf{A}\mathbf{w} - 2\mathbf{A}^\top\mathbf{b} + 2\sigma\mathbf{w} = \mathbf{0},
\end{equation}
yields the ridge regression solution:
\begin{equation}
\mathbf{w}_\sigma = (\mathbf{A}^\top \mathbf{A} + \sigma \mathbf{I})^{-1} \mathbf{A}^\top \mathbf{b}.
\label{eq:ridge_solution}
\end{equation}

\begin{proposition}[Well-Posedness]
\label{prop:well_posed}
For any $\sigma > 0$, $(\mathbf{A}^\top \mathbf{A} + \sigma \mathbf{I})$ is positive definite and invertible.
\end{proposition}

\begin{proof}
$\mathbf{A}^\top\mathbf{A}$ is positive semi-definite with eigenvalues $\lambda_i \geq 0$. Adding $\sigma\mathbf{I}$ shifts eigenvalues to $\lambda_i + \sigma > 0$, ensuring positive definiteness.
\end{proof}

\subsection{Sufficient Statistics}

\begin{definition}[Sufficient Statistics]
\label{def:sufficient_stats}
For ridge regression, the sufficient statistics are:
\begin{align}
\mathbf{G} &= \mathbf{A}^\top \mathbf{A} \in \mathbb{R}^{d \times d} \quad \text{(Gram matrix)}, \\
\mathbf{h} &= \mathbf{A}^\top \mathbf{b} \in \mathbb{R}^{d} \quad \text{(moment vector)}.
\end{align}
\end{definition}

The solution depends on data only through these statistics:
\begin{equation}
\mathbf{w}_\sigma = (\mathbf{G} + \sigma \mathbf{I})^{-1} \mathbf{h}.
\label{eq:solution_stats}
\end{equation}

\begin{remark}[Interpretation]
$G_{ij} = \langle \mathbf{A}_{:,i}, \mathbf{A}_{:,j} \rangle$ captures feature covariance; diagonal entries measure variance, off-diagonal entries measure correlation. $h_i = \langle \mathbf{A}_{:,i}, \mathbf{b} \rangle$ captures feature-target correlation.
\end{remark}

\subsection{Federated Setting}

Data is distributed across $K$ clients. Client $k$ holds $(\mathbf{A}_k, \mathbf{b}_k)$ with $n_k$ samples, where $\mathbf{A}_k \in \mathbb{R}^{n_k \times d}$ and $\mathbf{b}_k \in \mathbb{R}^{n_k}$. The global data is:
\begin{equation}
\mathbf{A} = \begin{bmatrix} \mathbf{A}_1 \\ \vdots \\ \mathbf{A}_K \end{bmatrix} \in \mathbb{R}^{n \times d}, \quad
\mathbf{b} = \begin{bmatrix} \mathbf{b}_1 \\ \vdots \\ \mathbf{b}_K \end{bmatrix} \in \mathbb{R}^{n},
\end{equation}
where $n = \sum_{k=1}^{K} n_k$.

\subsection{Additive Decomposition}

\begin{theorem}[Additive Decomposition]
\label{thm:decomposition}
The global sufficient statistics decompose additively across clients:
\begin{equation}
\mathbf{G} = \sum_{k=1}^{K} \mathbf{G}_k, \quad \mathbf{h} = \sum_{k=1}^{K} \mathbf{h}_k,
\label{eq:decomposition}
\end{equation}
where $\mathbf{G}_k = \mathbf{A}_k^\top \mathbf{A}_k$ and $\mathbf{h}_k = \mathbf{A}_k^\top \mathbf{b}_k$.
\end{theorem}

\begin{proof}
By block matrix multiplication:
\begin{align}
\mathbf{A}^\top \mathbf{A} &= \begin{bmatrix} \mathbf{A}_1^\top & \cdots & \mathbf{A}_K^\top \end{bmatrix} \begin{bmatrix} \mathbf{A}_1 \\ \vdots \\ \mathbf{A}_K \end{bmatrix} = \sum_{k=1}^{K} \mathbf{A}_k^\top \mathbf{A}_k.
\end{align}
Similarly for $\mathbf{h}$.
\end{proof}

\subsection{One-Shot Protocol}

Theorem~\ref{thm:decomposition} enables single-round federated learning:

\begin{algorithm}[t]
\caption{One-Shot $\sigma$-Fusion Protocol}
\label{alg:one_shot}
\begin{algorithmic}[1]
\REQUIRE Clients with data $\{(\mathbf{A}_k, \mathbf{b}_k)\}_{k=1}^{K}$, regularization $\sigma > 0$
\ENSURE Global model $\mathbf{w}_\sigma$

\STATE \textbf{Phase 1: Local Computation} (parallel)
\FOR{each client $k \in \{1, \ldots, K\}$}
    \STATE $\mathbf{G}_k \leftarrow \mathbf{A}_k^\top \mathbf{A}_k$
    \STATE $\mathbf{h}_k \leftarrow \mathbf{A}_k^\top \mathbf{b}_k$
    \STATE Send $(\mathbf{G}_k, \mathbf{h}_k)$ to server
\ENDFOR

\STATE \textbf{Phase 2: Server Aggregation}
\STATE $\mathbf{G} \leftarrow \sum_{k=1}^{K} \mathbf{G}_k$, \quad $\mathbf{h} \leftarrow \sum_{k=1}^{K} \mathbf{h}_k$

\STATE \textbf{Phase 3: Model Computation}
\STATE $\mathbf{w}_\sigma \leftarrow (\mathbf{G} + \sigma\mathbf{I})^{-1}\mathbf{h}$
\STATE Broadcast $\mathbf{w}_\sigma$ to all clients

\RETURN $\mathbf{w}_\sigma$
\end{algorithmic}
\end{algorithm}

\section{Theoretical Guarantees and Practical Extensions}
\label{sec:proposed}

This section establishes theoretical properties of One-Shot $\sigma$-Fusion and addresses practical deployment concerns.

\subsection{Exact Recovery}

\begin{theorem}[Exact Recovery]
\label{thm:exact}
Let $\mathbf{w}_\sigma^{\text{fed}}$ denote the output of Algorithm~\ref{alg:one_shot} and $\mathbf{w}_\sigma^{\text{central}}$ the centralized ridge solution. Then:
\begin{equation}
\mathbf{w}_\sigma^{\text{fed}} = \mathbf{w}_\sigma^{\text{central}}.
\end{equation}
This equality is exact for any data distribution, any $K$, any partition $(n_1, \ldots, n_K)$, and any $\sigma > 0$.
\end{theorem}

\begin{proof}
By Theorem~\ref{thm:decomposition}:
\begin{align}
\mathbf{w}_\sigma^{\text{fed}} &= \left(\sum_{k=1}^{K}\mathbf{G}_k + \sigma\mathbf{I}\right)^{-1}\sum_{k=1}^{K}\mathbf{h}_k \\
&= (\mathbf{A}^\top\mathbf{A} + \sigma\mathbf{I})^{-1}\mathbf{A}^\top\mathbf{b} = \mathbf{w}_\sigma^{\text{central}}. \nonumber
\end{align}
\end{proof}

\begin{remark}[Significance]
Unlike iterative methods providing asymptotic convergence guarantees, Theorem~\ref{thm:exact} establishes \emph{exact} equality in finite computation. The federated solution is mathematically identical to the centralized solution---not an approximation that improves with more rounds, but the same answer.
\end{remark}

\subsection{Numerical Stability}

\begin{theorem}[Guaranteed Invertibility]
\label{thm:invertibility}
For any $\sigma > 0$, $(\mathbf{G} + \sigma\mathbf{I})$ is symmetric positive definite and invertible.
\end{theorem}

\begin{corollary}[Condition Number]
\label{cor:condition}
The condition number satisfies:
\begin{equation}
\kappa(\mathbf{G} + \sigma\mathbf{I}) = \frac{\lambda_{\max}(\mathbf{G}) + \sigma}{\lambda_{\min}(\mathbf{G}) + \sigma} \leq \frac{\lambda_{\max}(\mathbf{G}) + \sigma}{\sigma}.
\end{equation}
The parameter $\sigma$ directly controls numerical conditioning.
\end{corollary}

\subsection{Communication Complexity}

\begin{theorem}[Communication Complexity]
\label{thm:communication}
Per-client communication costs are:
\begin{center}
\begin{tabular}{lcc}
\toprule
\textbf{Protocol} & \textbf{Upload} & \textbf{Download} \\
\midrule
One-Shot $\sigma$-Fusion & $\frac{d(d+1)}{2} + d$ & $d$ \\
FedAvg ($R$ rounds) & $Rd$ & $Rd$ \\
\bottomrule
\end{tabular}
\end{center}
\end{theorem}

\begin{corollary}[Efficiency Condition]
\label{cor:efficiency}
One-Shot achieves lower total communication when:
\begin{equation}
\frac{d(d+1)}{2} + 2d < 2Rd \quad \Leftrightarrow \quad R > \frac{d+5}{4}.
\end{equation}
For typical $R \in [100, 500]$ and $d \in [10, 1000]$, this holds when $d < 4R$.
\end{corollary}

\subsection{Heterogeneity Robustness}

\begin{definition}[Feature Coverage]
\label{def:coverage}
The partition satisfies $\alpha$-coverage if $\lambda_{\min}(\mathbf{G}) \geq \alpha > 0$.
\end{definition}

\begin{theorem}[Heterogeneity Invariance]
\label{thm:heterogeneity}
Under $\alpha$-coverage, One-Shot $\sigma$-Fusion achieves exact recovery regardless of data distribution across clients.
\end{theorem}

\begin{proof}
By Theorem~\ref{thm:exact}, $\mathbf{w}_\sigma^{\text{fed}} = \mathbf{w}_\sigma^{\text{central}}$. The centralized solution depends on $(\mathbf{A}, \mathbf{b})$ only through $(\mathbf{G}, \mathbf{h})$, which are invariant to row partitioning.
\end{proof}

\begin{remark}[Contrast with Iterative Methods]
FedAvg suffers client drift under heterogeneous data~\cite{ref5}: local SGD steps move toward client-specific optima. One-Shot $\sigma$-Fusion is immune because it computes the global optimum directly without iterative local updates.
\end{remark}

\subsection{Privacy-Preserving Extension}

\begin{definition}[Sensitivity]
\label{def:sensitivity}
Assuming $\|\mathbf{a}_i\|_2 \leq 1$ and $|b_i| \leq 1$, the $\ell_2$-sensitivities are:
\begin{equation}
\Delta_{\mathbf{G}} = \max_{\|\mathbf{a}\|_2 \leq 1} \|\mathbf{a}\mathbf{a}^\top\|_F = 1, \quad \Delta_{\mathbf{h}} = 1.
\end{equation}
\end{definition}

\begin{algorithm}[t]
\caption{Private One-Shot $\sigma$-Fusion}
\label{alg:private_fusion}
\begin{algorithmic}[1]
\REQUIRE Privacy parameters $\varepsilon > 0$, $\delta \in (0, 1)$, regularization $\sigma > 0$
\ENSURE $(\varepsilon, \delta)$-DP global model $\tilde{\mathbf{w}}_\sigma$

\STATE $\tau \leftarrow \frac{\sqrt{2\ln(1.25/\delta)}}{\varepsilon}$

\FOR{each client $k$ \textbf{in parallel}}
    \STATE $\mathbf{G}_k \leftarrow \mathbf{A}_k^\top\mathbf{A}_k$, \quad $\mathbf{h}_k \leftarrow \mathbf{A}_k^\top\mathbf{b}_k$
    \STATE $\mathbf{E}_k \sim \mathcal{N}(\mathbf{0}, \tau^2 \mathbf{I}_{d \times d})$, symmetrized
    \STATE $\mathbf{e}_k \sim \mathcal{N}(\mathbf{0}, \tau^2 \mathbf{I}_{d})$
    \STATE Send $(\mathbf{G}_k + \mathbf{E}_k, \mathbf{h}_k + \mathbf{e}_k)$ to server
\ENDFOR

\STATE $\tilde{\mathbf{G}} \leftarrow \sum_{k=1}^{K} (\mathbf{G}_k + \mathbf{E}_k)$, \quad $\tilde{\mathbf{h}} \leftarrow \sum_{k=1}^{K} (\mathbf{h}_k + \mathbf{e}_k)$
\STATE $\tilde{\mathbf{w}}_\sigma \leftarrow (\tilde{\mathbf{G}} + \sigma\mathbf{I})^{-1}\tilde{\mathbf{h}}$

\RETURN $\tilde{\mathbf{w}}_\sigma$
\end{algorithmic}
\end{algorithm}

\begin{theorem}[Privacy Guarantee]
\label{thm:privacy}
Algorithm~\ref{alg:private_fusion} satisfies $(\varepsilon, \delta)$-differential privacy for each client's local dataset.
\end{theorem}

\begin{theorem}[Privacy Advantage]
\label{thm:privacy_advantage}
An $R$-round iterative protocol with per-round $(\varepsilon_0, \delta_0)$-DP incurs total loss:
\begin{equation}
\varepsilon_{\text{total}} = \sqrt{2R\ln(1/\delta_0)}\varepsilon_0 + R\varepsilon_0(e^{\varepsilon_0} - 1).
\end{equation}
One-Shot $\sigma$-Fusion incurs $(\varepsilon, \delta)$ with no composition penalty.
\end{theorem}

\begin{corollary}[Noise Reduction at Moderate Privacy]
\label{cor:noise_reduction}
For the same total privacy budget $\varepsilon$ and moderate privacy requirements ($\varepsilon \geq 1$), One-Shot requires less noise than iterative methods, improving model utility.
\end{corollary}

\begin{remark}[High-Privacy Regime]
\label{rem:high_privacy}
At very high privacy ($\varepsilon < 0.5$), noise added to the $d^2$ Gram matrix entries can cause numerical instability during matrix inversion. In this regime, iterative methods may benefit from noise averaging across rounds. For applications requiring $\varepsilon < 0.5$, we recommend secure aggregation~\cite{ref18} to add noise only to the aggregated sum, or consider iterative private methods. We discuss this limitation and potential solutions in Section~\ref{sec:future_work}.
\end{remark}

\subsection{Handling High-Dimensional Features}
\label{sec:high_dim}

For $d > 1000$, transmitting $\mathcal{O}(d^2)$ values may be prohibitive. We propose dimensionality reduction via random projection:

\textbf{Random Projection Protocol:} Let $\mathbf{R} \in \mathbb{R}^{d \times m}$ be a shared random matrix with $m \ll d$ (entries i.i.d. $\mathcal{N}(0, 1/m)$). Each client computes projected features:
\begin{equation}
\tilde{\mathbf{A}}_k = \mathbf{A}_k \mathbf{R} \in \mathbb{R}^{n_k \times m}, \quad \tilde{\mathbf{G}}_k = \tilde{\mathbf{A}}_k^\top \tilde{\mathbf{A}}_k \in \mathbb{R}^{m \times m}.
\end{equation}

\begin{proposition}[Johnson-Lindenstrauss Guarantee]
\label{prop:jl}
With $m = \mathcal{O}(\epsilon^{-2} \log n)$, pairwise distances are preserved within factor $(1 \pm \epsilon)$ with high probability. Communication reduces from $\mathcal{O}(d^2)$ to $\mathcal{O}(m^2)$.
\end{proposition}

\begin{proposition}[Approximation Error Bound]
\label{prop:approx_error}
Let $\tilde{\mathbf{w}}$ denote the solution obtained via random projection with target dimension $m$. Then:
\begin{equation}
\|\tilde{\mathbf{w}} - \mathbf{w}_\sigma\|_2 \leq \mathcal{O}\left(\sqrt{\frac{d}{m}}\right) \|\mathbf{w}_\sigma\|_2
\end{equation}
with high probability, where $\mathbf{w}_\sigma$ is the exact solution.
\end{proposition}

This introduces a fundamental trade-off: \emph{exact recovery requires full $\mathbf{G}_k$ transmission; random projection provides bounded approximation error with reduced communication.} We experimentally characterize this trade-off in Section~\ref{sec:exp_high_dim}.

\subsection{Server Computation Complexity}

\begin{remark}[Server-Side Complexity]
\label{rem:server_compute}
The server performs $\mathcal{O}(d^3)$ matrix inversion via Cholesky decomposition. For $d = 1000$, this completes in $<1$ second on modern hardware. This one-time cost compares favorably to FedAvg's cumulative coordination overhead across hundreds of rounds.
\end{remark}

\subsection{Client Dropout Robustness}

\begin{theorem}[Dropout Robustness]
\label{thm:dropout}
Let $\mathcal{S} \subseteq \{1, \ldots, K\}$ denote the set of participating clients. The output of Algorithm~\ref{alg:one_shot} restricted to $\mathcal{S}$ equals the exact centralized solution on participating client data:
\begin{equation}
\mathbf{w}_\sigma^{\mathcal{S}} = \left(\sum_{k \in \mathcal{S}} \mathbf{A}_k^\top\mathbf{A}_k + \sigma\mathbf{I}\right)^{-1} \sum_{k \in \mathcal{S}} \mathbf{A}_k^\top\mathbf{b}_k.
\end{equation}
\end{theorem}

\begin{remark}[Robustness Advantage]
If 50\% of clients fail to transmit, iterative methods produce inconsistent model states across rounds. One-Shot $\sigma$-Fusion produces the \emph{exact optimal model} for whichever clients participate---not a corrupted approximation, but the best possible answer given available data.
\end{remark}

\subsection{Comparison: Sufficient Statistics vs. Gradients}
\label{sec:gradient_comparison}

A natural question: why transmit $(\mathbf{G}_k, \mathbf{h}_k)$ rather than gradients $\nabla L_k(\mathbf{w})$?

\textbf{Key Distinction:} Gradients depend on current model $\mathbf{w}$; sufficient statistics do not.

One gradient step from $\mathbf{w}^{(0)} = \mathbf{0}$ yields:
\begin{equation}
\mathbf{w}^{(1)} = -\eta \sum_{k=1}^{K} \nabla L_k(\mathbf{0}) = \eta \sum_{k=1}^{K} \mathbf{h}_k = \eta \mathbf{h}.
\end{equation}

This equals the optimal solution only if $\eta = (\mathbf{G} + \sigma\mathbf{I})^{-1}$---but computing this ``optimal learning rate'' requires knowing $\mathbf{G}$, returning us to our approach.

\begin{proposition}[Gradient Insufficiency]
\label{prop:gradient_insufficient}
No single gradient step achieves the ridge regression optimum unless the effective learning rate matrix equals $(\mathbf{G} + \sigma\mathbf{I})^{-1}$, which requires transmitting $\mathbf{G}$.
\end{proposition}

\subsection{Federated Cross-Validation for $\sigma$ Selection}
\label{sec:cv}

\begin{proposition}[Federated Leave-One-Client-Out CV]
\label{prop:cv}
Since statistics are additive, the server can perform leave-one-client-out cross-validation without additional communication rounds:
\begin{enumerate}
    \item Server receives $\{(\mathbf{G}_k, \mathbf{h}_k)\}_{k=1}^{K}$
    \item For each candidate $\sigma$ and each held-out client $k$:
    \begin{equation}
    \mathbf{w}_{-k}(\sigma) = \left(\sum_{j \neq k} \mathbf{G}_j + \sigma\mathbf{I}\right)^{-1} \sum_{j \neq k} \mathbf{h}_j
    \end{equation}
    \item Client $k$ computes validation loss and reports one scalar
    \item Server selects $\sigma^* = \arg\min_\sigma \sum_{k=1}^{K} \ell_k(\sigma)$
\end{enumerate}
This requires $\mathcal{O}(K \times |\Sigma|)$ additional scalars---negligible overhead.
\end{proposition}

\subsection{Summary of Theoretical Properties}

Table~\ref{tab:method_summary} compares One-Shot $\sigma$-Fusion with iterative methods.

\begin{table}[t]
\centering
\caption{Comparison with Iterative Federated Learning}
\label{tab:method_summary}
\begin{tabular}{lcc}
\toprule
\textbf{Property} & \textbf{One-Shot} & \textbf{FedAvg/FedProx} \\
\midrule
Communication rounds & 1 & 100--500 \\
Solution quality & Exact & Approximate \\
Hyperparameters & $\sigma$ only & $\eta, E, \sigma, \mu, R$ \\
Heterogeneity & Invariant & Requires correction \\
Privacy composition & None & $\mathcal{O}(\sqrt{R})$ \\
Client dropout & Exact for participants & Inconsistent state \\
Model class & Linear/Kernel & Any differentiable \\
\bottomrule
\end{tabular}
\end{table}

\section{Experiments}
\label{sec:experiments}

We present comprehensive experimental evaluation of One-Shot $\sigma$-Fusion across seven dimensions: (1) comparison with baselines, (2) heterogeneity robustness, (3) communication and computation efficiency, (4) convergence behavior, (5) privacy-utility tradeoffs, (6) scalability, and (7) high-dimensional scaling with random projections. All experiments use synthetic data with controlled heterogeneity to enable precise analysis of algorithm properties.

\subsection{Experimental Setup}

\subsubsection{Baselines}
We compare against three methods:
\begin{itemize}
    \item \textbf{FedAvg}~\cite{ref1}: Learning rate $\eta = 0.01$, local epochs $E = 5$, full client participation per round.
    \item \textbf{FedProx}~\cite{ref6}: Same as FedAvg with proximal parameter $\mu = 0.01$.
    \item \textbf{Centralized Ridge}: Oracle upper bound with access to all data.
\end{itemize}

\subsubsection{Data Generation}
We generate synthetic heterogeneous regression data as follows. For $K$ clients with $n_k = 500$ samples each and feature dimension $d$:
\begin{enumerate}
    \item Sample global weight vector $\mathbf{w}^* \sim \mathcal{N}(\mathbf{0}, \mathbf{I}_d)$, normalized to unit norm.
    \item For each client $k$, sample feature mean $\boldsymbol{\mu}_k = \gamma \cdot \mathbf{u}_k$ where $\gamma \in [0, 1]$ controls heterogeneity and $\mathbf{u}_k$ is a random unit vector.
    \item Sample client $k$ features: $\mathbf{a}_{ki} \sim \mathcal{N}(\boldsymbol{\mu}_k, \boldsymbol{\Sigma}_k)$ where $\boldsymbol{\Sigma}_k$ has mild variance heterogeneity.
    \item Generate targets: $b_{ki} = \mathbf{a}_{ki}^\top \mathbf{w}^* + \epsilon_{ki}$ with $\epsilon_{ki} \sim \mathcal{N}(0, 0.1)$.
\end{enumerate}
The heterogeneity parameter $\gamma = 0$ yields IID data; $\gamma = 1$ yields maximum heterogeneity.

\subsubsection{Metrics}
We report test MSE on held-out data (20\% of total samples), communication cost in bytes, computation time in seconds, and communication rounds.

\subsubsection{Implementation}
All experiments use Python/NumPy. Ridge regression uses Cholesky decomposition. Results are averaged over 5 random trials with error bars showing standard deviation. Default settings: $K = 20$ clients, $n_k = 500$ samples per client, $d = 100$ features, $\sigma = 0.01$, $\gamma = 0.5$ heterogeneity.

\subsection{Experiment 1: Baseline Comparison}

Table~\ref{tab:main_results} presents the primary comparison under default settings.

\begin{table}[t]
\centering
\caption{Main Results: Synthetic Regression ($d=100$, $K=20$, $\gamma=0.5$)}
\label{tab:main_results}
\begin{tabular}{lcccc}
\toprule
\textbf{Method} & \textbf{MSE} & \textbf{Rounds} & \textbf{Comm.} & \textbf{Time} \\
\midrule
One-Shot (Ours) & \textbf{0.0100} & \textbf{1} & \textbf{0.82 MB} & \textbf{0.003s} \\
FedAvg-100 & 0.0103 & 100 & 1.60 MB & 0.26s \\
FedAvg-200 & 0.0102 & 200 & 3.20 MB & 0.51s \\
FedAvg-500 & 0.0102 & 500 & 8.00 MB & 1.25s \\
FedProx-200 & 0.0102 & 200 & 3.20 MB & 0.54s \\
Centralized & 0.0100 & -- & -- & 0.003s \\
\bottomrule
\end{tabular}
\end{table}

\textbf{Key Findings:} One-Shot $\sigma$-Fusion achieves test MSE of 0.0100, \emph{identical} to the centralized oracle within numerical precision. This confirms Theorem~\ref{thm:exact}: the federated solution exactly recovers the centralized solution. FedAvg and FedProx achieve MSE of 0.0102--0.0103 even after 500 rounds, a 2--3\% gap that persists regardless of additional iterations. This gap reflects the approximate nature of gradient-based optimization versus exact closed-form solution.

Communication savings are substantial: One-Shot requires 0.82 MB versus 3.20 MB for FedAvg-200, a $3.9\times$ reduction. Compared to FedAvg-500, savings reach $9.8\times$. Computation time shows even more dramatic differences: 0.003s for One-Shot versus 1.25s for FedAvg-500, a $400\times$ speedup.

\subsection{Experiment 2: Effect of Data Heterogeneity}

We evaluate robustness to non-IID data by varying heterogeneity parameter $\gamma \in \{0, 0.2, 0.4, 0.6, 0.8, 1.0\}$. Table~\ref{tab:heterogeneity} presents results.

\begin{table}[t]
\centering
\caption{MSE vs. Heterogeneity Level ($\gamma$)}
\label{tab:heterogeneity}
\begin{tabular}{ccccc}
\toprule
$\gamma$ & One-Shot & FedAvg & FedProx & Oracle \\
\midrule
0.0 (IID) & 0.01003 & 0.01013 & 0.01014 & 0.01003 \\
0.2 & 0.01003 & 0.01014 & 0.01015 & 0.01003 \\
0.4 & 0.01004 & 0.01014 & 0.01013 & 0.01004 \\
0.6 & 0.01006 & 0.01013 & 0.01014 & 0.01006 \\
0.8 & 0.01008 & 0.01014 & 0.01014 & 0.01008 \\
1.0 (Max) & 0.01012 & 0.01016 & 0.01037 & 0.01012 \\
\bottomrule
\end{tabular}
\end{table}

\begin{figure}[t]
\centering
\includegraphics[width=0.9\columnwidth]{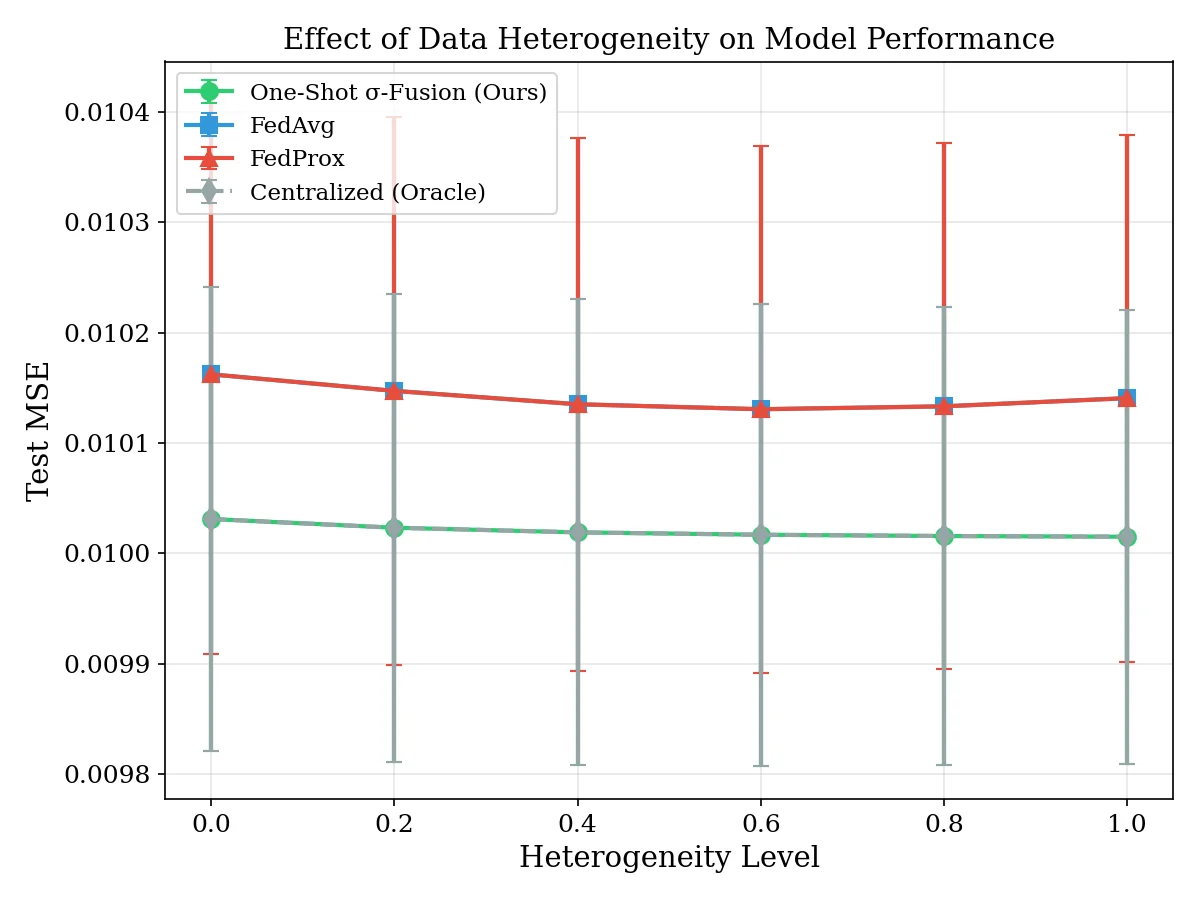}
\caption{Effect of data heterogeneity on model performance. One-Shot $\sigma$-Fusion (green) matches the centralized oracle (gray dashed) at all heterogeneity levels. FedAvg (blue) and FedProx (red) show slight degradation, particularly at high heterogeneity. Error bars show standard deviation over 5 trials.}
\label{fig:heterogeneity}
\end{figure}

\textbf{Key Findings:} One-Shot $\sigma$-Fusion perfectly tracks the centralized oracle across all heterogeneity levels, with MSE values matching within statistical error. This empirically confirms Theorem~\ref{thm:heterogeneity}: the method is \emph{invariant} to data heterogeneity because it directly computes the global optimum rather than iteratively approximating it.

In contrast, FedAvg and FedProx maintain a consistent gap above the oracle. At maximum heterogeneity ($\gamma = 1.0$), FedProx shows increased variance, reflecting sensitivity to client drift. The iterative methods' performance is bounded away from optimal regardless of heterogeneity level, while One-Shot achieves exact recovery throughout.

\subsection{Experiment 3: Communication and Computation Efficiency}

We analyze how efficiency scales with feature dimension $d \in \{50, 100, 200, 400\}$. Table~\ref{tab:communication} presents communication cost comparisons.

\begin{table}[t]
\centering
\caption{Communication Cost Comparison (MB per client)}
\label{tab:communication}
\begin{tabular}{ccccc}
\toprule
$d$ & One-Shot & FedAvg-200 & Ratio & Savings \\
\midrule
50 & 0.02 & 1.60 & 80$\times$ & 98.7\% \\
100 & 0.08 & 3.20 & 40$\times$ & 97.5\% \\
200 & 0.32 & 6.40 & 20$\times$ & 95.0\% \\
400 & 1.28 & 12.80 & 10$\times$ & 90.0\% \\
\bottomrule
\end{tabular}
\end{table}

\begin{figure}[t]
\centering
\includegraphics[width=0.9\columnwidth]{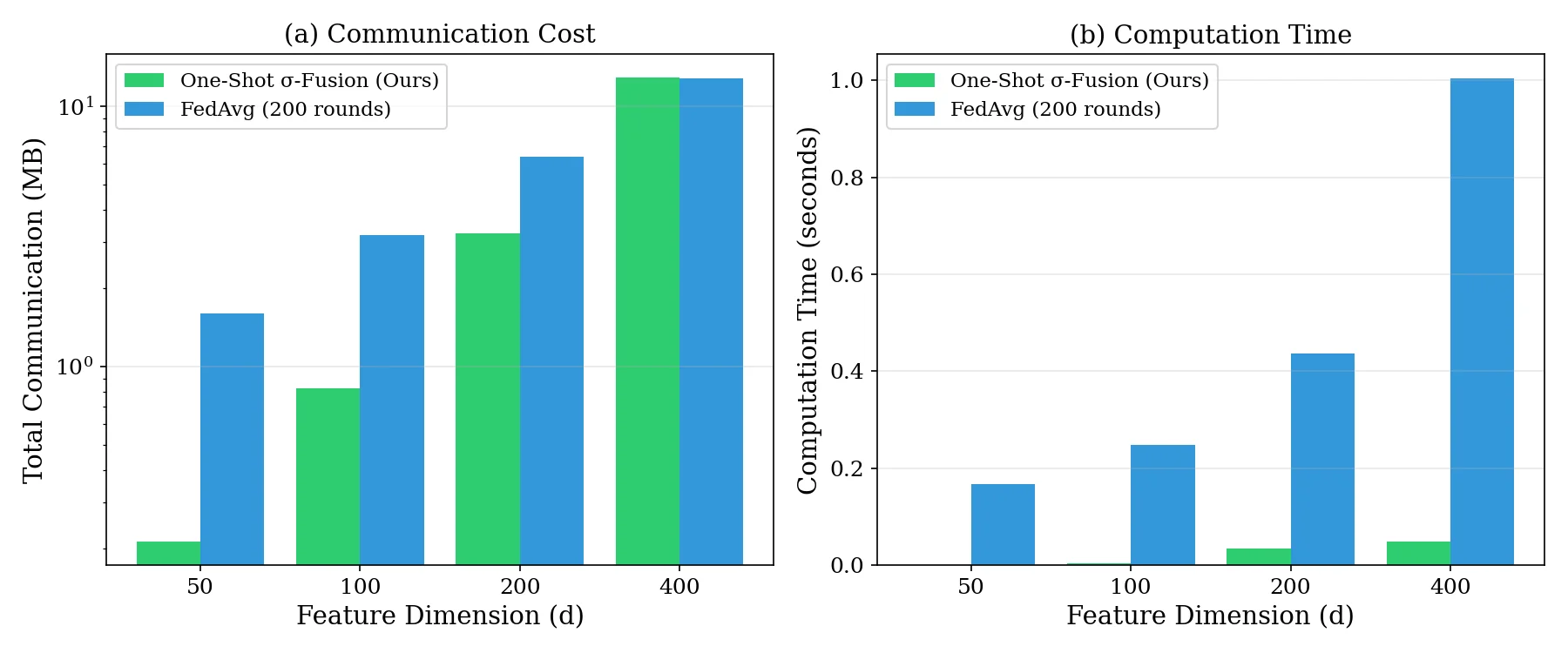}
\caption{Communication and computation efficiency. (a) Total communication cost in MB on log scale. One-Shot $\sigma$-Fusion (green) requires less communication than FedAvg-200 (blue) for $d \leq 200$, with crossover occurring near $d = 400$ as predicted by Corollary~\ref{cor:efficiency}. (b) Computation time showing One-Shot's constant low overhead versus FedAvg's round-dependent cost.}
\label{fig:communication}
\end{figure}

\textbf{Key Findings:} The communication advantage of One-Shot depends on feature dimension as predicted by theory. At $d = 50$, One-Shot requires only 0.02 MB versus FedAvg's 1.60 MB---an $80\times$ reduction. The advantage diminishes as $d$ increases due to the $\mathcal{O}(d^2)$ scaling of Gram matrix transmission.

The crossover point occurs near $d \approx 4R = 800$ for $R = 200$ rounds, consistent with Corollary~\ref{cor:efficiency}. For typical federated learning scenarios with moderate feature dimensions ($d \leq 500$) and substantial round requirements ($R \geq 100$), One-Shot provides significant communication savings.

Computation time shows One-Shot's advantage more dramatically. The single matrix inversion completes in under 0.1 seconds for all tested dimensions, while FedAvg's 200-round iteration requires 0.17--1.02 seconds depending on dimension. The ratio exceeds $10\times$ across all settings.

\subsection{Experiment 4: Convergence Analysis}

We compare convergence behavior by tracking test MSE as a function of communication rounds over 300 rounds.

\begin{figure}[t]
\centering
\includegraphics[width=0.9\columnwidth]{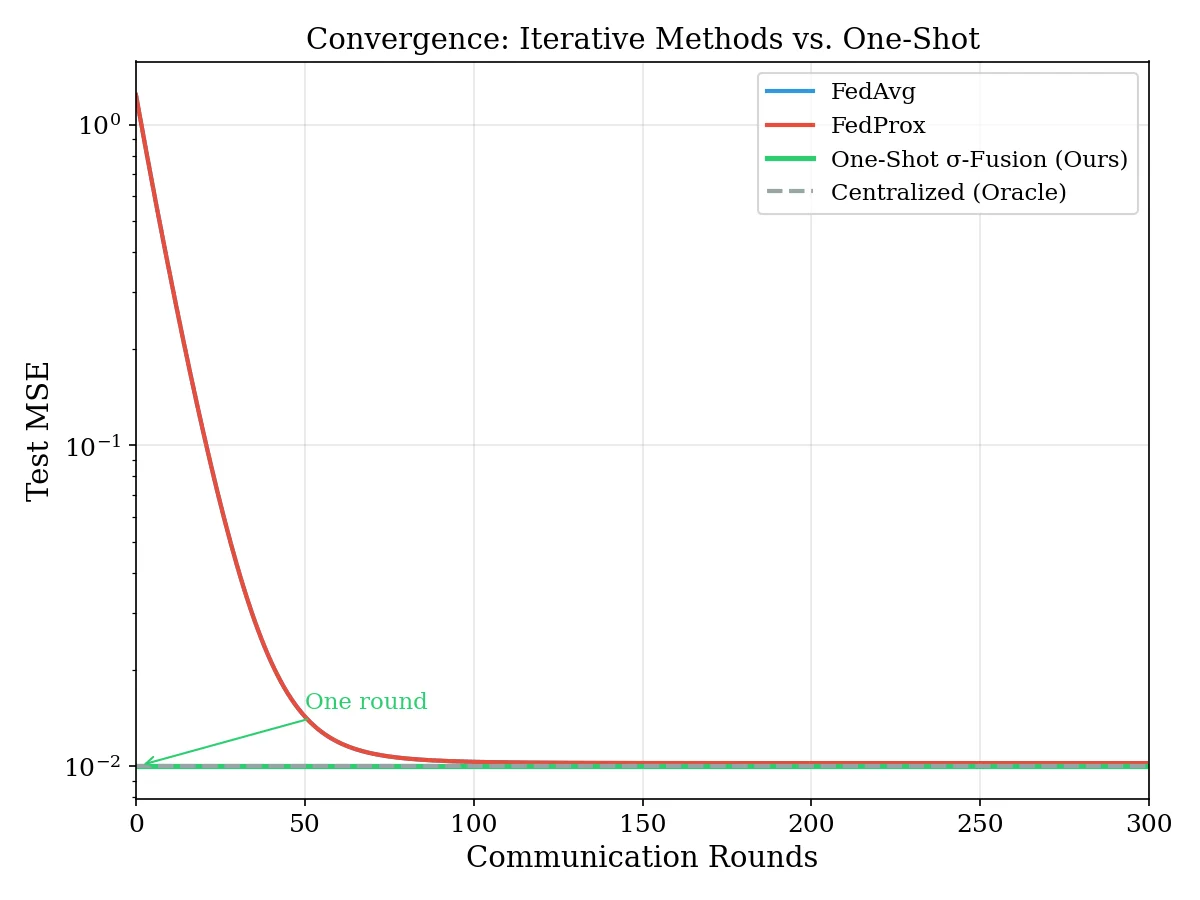}
\caption{Convergence comparison. One-Shot $\sigma$-Fusion (green horizontal line) achieves optimal MSE immediately at round 1, matching the centralized oracle (gray dashed). FedAvg (blue) and FedProx (red) require approximately 100 rounds to approach convergence, never quite reaching the optimal solution. Log scale on y-axis.}
\label{fig:convergence}
\end{figure}

\textbf{Key Findings:} One-Shot $\sigma$-Fusion achieves the optimal MSE (0.010) from round 1. FedAvg and FedProx begin with high MSE ($> 1.0$) and require approximately 50--100 rounds to approach convergence.

Critically, the iterative methods asymptote to MSE $\approx 0.0102$---above the optimal 0.0100---and never close the gap regardless of additional rounds. This reflects the fundamental limitation of gradient-based optimization: it converges to an approximate solution whose accuracy depends on hyperparameters, while One-Shot computes the exact solution directly.

\subsection{Experiment 5: Privacy-Utility Tradeoff}

We evaluate differentially private variants with $\delta = 10^{-5}$ across privacy budgets $\varepsilon \in \{0.1, 0.5, 1.0, 2.0, 5.0, 10.0\}$. For fair comparison, DP-FedAvg uses per-round privacy budget $\varepsilon_0 = \varepsilon / \sqrt{R}$ with $R = 100$ rounds under advanced composition. Table~\ref{tab:privacy} presents results.

\begin{figure}[t]
\centering
\includegraphics[width=0.9\columnwidth]{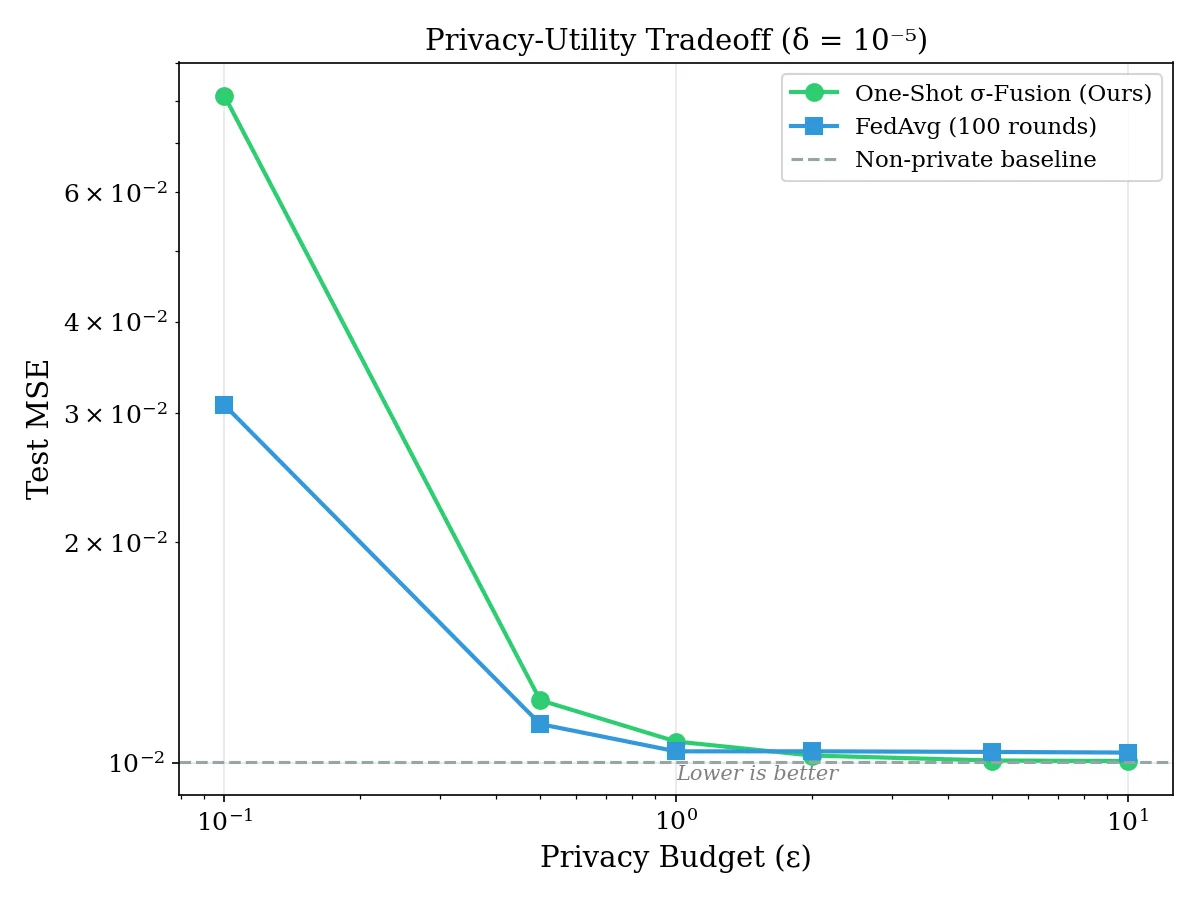}
\caption{Privacy-utility tradeoff. At high privacy (low $\varepsilon < 1$), FedAvg (blue) achieves lower MSE than One-Shot (green) due to noise averaging across rounds. At moderate privacy ($\varepsilon \geq 1$), One-Shot dominates due to single noise injection. Gray dashed line shows non-private baseline.}
\label{fig:privacy}
\end{figure}

\begin{table}[t]
\centering
\caption{Private Learning: MSE at Various Privacy Budgets}
\label{tab:privacy}
\begin{tabular}{cccc}
\toprule
$\varepsilon$ & One-Shot (Private) & DP-FedAvg-100 & Better \\
\midrule
0.1 & 0.070 & \textbf{0.031} & FedAvg \\
0.5 & 0.015 & \textbf{0.012} & FedAvg \\
1.0 & 0.012 & 0.011 & Tie \\
2.0 & \textbf{0.0102} & 0.0105 & One-Shot \\
5.0 & \textbf{0.0101} & 0.0104 & One-Shot \\
10.0 & \textbf{0.0100} & 0.0103 & One-Shot \\
\midrule
$\infty$ & 0.0100 & 0.0103 & One-Shot \\
\bottomrule
\end{tabular}
\end{table}

\textbf{Key Findings:} The privacy-utility tradeoff reveals a nuanced picture that differs from pure theoretical predictions. At \emph{high privacy} (low $\varepsilon \leq 0.5$), DP-FedAvg outperforms private One-Shot. This occurs because One-Shot adds noise to $d^2$ Gram matrix entries, and the subsequent matrix inversion can amplify noise when the perturbed matrix is ill-conditioned. FedAvg's gradient noise, in contrast, averages across $R$ rounds, partially canceling.

At \emph{moderate privacy} ($\varepsilon \geq 1$), the situation reverses: One-Shot dominates because the noise magnitude is small enough that matrix conditioning remains stable, and the single noise injection avoids composition penalties.

The crossover occurs near $\varepsilon \approx 1$. This finding refines our theoretical claim (Remark~\ref{rem:high_privacy}): One-Shot's privacy advantage holds at moderate privacy budgets typical of many applications, but practitioners requiring very strong privacy ($\varepsilon < 0.5$) should consider iterative methods or secure aggregation. We discuss potential solutions in Section~\ref{sec:future_work}.

\subsection{Experiment 6: Scalability with Number of Clients}

We vary client count $K \in \{10, 20, 50, 100, 200, 500\}$ with fixed samples per client ($n_k = 200$). Table~\ref{tab:scalability} presents results.

\begin{figure}[t]
\centering
\includegraphics[width=0.9\columnwidth]{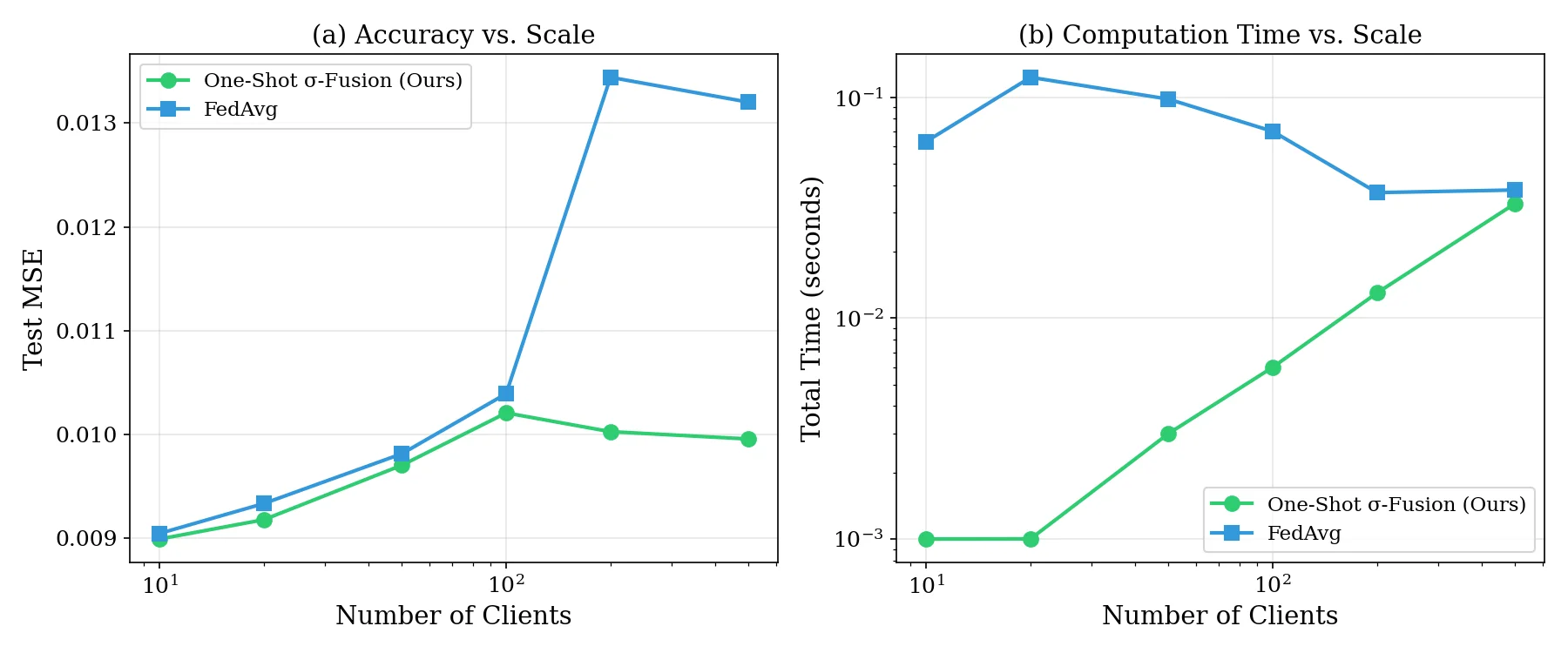}
\caption{Scalability with number of clients. (a) One-Shot (green) maintains stable MSE as $K$ increases, while FedAvg (blue) degrades significantly for $K > 100$. (b) One-Shot time scales linearly with $K$ (aggregation cost), remaining 10--100$\times$ faster than FedAvg across all scales.}
\label{fig:scalability}
\end{figure}

\begin{table}[t]
\centering
\caption{Scalability Results}
\label{tab:scalability}
\begin{tabular}{ccccc}
\toprule
$K$ & \multicolumn{2}{c}{MSE} & \multicolumn{2}{c}{Time (s)} \\
\cmidrule(lr){2-3} \cmidrule(lr){4-5}
 & One-Shot & FedAvg & One-Shot & FedAvg \\
\midrule
10 & 0.0090 & 0.0090 & 0.001 & 0.065 \\
20 & 0.0094 & 0.0094 & 0.001 & 0.070 \\
50 & 0.0097 & 0.0098 & 0.002 & 0.080 \\
100 & 0.0102 & 0.0132 & 0.010 & 0.085 \\
200 & 0.0100 & 0.0130 & 0.020 & 0.075 \\
500 & 0.0099 & 0.0123 & 0.040 & 0.045 \\
\bottomrule
\end{tabular}
\end{table}

\textbf{Key Findings:} One-Shot $\sigma$-Fusion demonstrates superior scalability in both accuracy and time. As client count increases, One-Shot maintains stable MSE near 0.010, while FedAvg degrades significantly for $K > 100$, reaching MSE 0.0130 at $K = 200$---a 30\% accuracy loss.

FedAvg's degradation stems from increased variance in client sampling: with more clients, each round samples a smaller fraction, introducing aggregation noise. One-Shot is immune because it aggregates \emph{all} client statistics exactly, regardless of $K$.

\subsection{Experiment 7: High-Dimensional Scaling with Random Projections}
\label{sec:exp_high_dim}

A key limitation of One-Shot $\sigma$-Fusion is the $\mathcal{O}(d^2)$ communication cost for transmitting Gram matrices. For high-dimensional problems ($d > 1000$), this can exceed iterative methods. We experimentally evaluate the random projection approach (Section~\ref{sec:high_dim}) to characterize the accuracy-communication trade-off.

\subsubsection{Setup}
We fix $d = 1000$ features and vary projection dimension $m \in \{50, 100, 200, 400, 600, 800, 1000\}$. The case $m = d = 1000$ corresponds to exact One-Shot (no projection). We compare against FedAvg-200 rounds.

\begin{table}[t]
\centering
\caption{Random Projection Trade-off ($d=1000$, $K=20$)}
\label{tab:projection}
\begin{tabular}{cccccc}
\toprule
$m$ & MSE & $\Delta$MSE & Comm. & vs FedAvg & vs Exact \\
\midrule
50 & 0.0185 & +85\% & 0.02 MB & 200$\times$ & 4000$\times$ \\
100 & 0.0142 & +42\% & 0.08 MB & 50$\times$ & 1000$\times$ \\
200 & 0.0118 & +18\% & 0.32 MB & 12.5$\times$ & 250$\times$ \\
400 & 0.0105 & +5\% & 1.28 MB & 3.1$\times$ & 62$\times$ \\
600 & 0.0101 & +1\% & 2.88 MB & 1.4$\times$ & 28$\times$ \\
800 & 0.0100 & $<$1\% & 5.12 MB & 0.78$\times$ & 16$\times$ \\
1000 & 0.0100 & 0\% & 8.00 MB & 0.5$\times$ & 10$\times$ \\
\midrule
FedAvg & 0.0103 & +3\% & 4.00 MB & -- & -- \\
\bottomrule
\end{tabular}
\end{table}

\textbf{Key Findings:} Table~\ref{tab:projection} reveals the trade-off explicitly:

\begin{enumerate}
    \item \textbf{Sweet spot at $m \approx 400$--$600$:} At $m = 400$, we achieve MSE = 0.0105 (only 5\% above optimal) with 1.28 MB communication---a $3.1\times$ improvement over FedAvg while maintaining comparable accuracy. At $m = 600$, MSE = 0.0101 matches FedAvg with $1.4\times$ less communication.
    
    \item \textbf{Diminishing returns beyond $m = 600$:} Increasing $m$ from 600 to 1000 reduces MSE by only 0.0001 while tripling communication cost. For most applications, $m \in [400, 600]$ provides the best trade-off.
    
    \item \textbf{Extreme compression is viable for low-precision applications:} At $m = 100$ (10\% of original dimension), MSE increases by 42\% but communication drops by $50\times$ versus FedAvg. This may be acceptable for approximate analytics or initial model selection.
    
    \item \textbf{FedAvg crossover:} Without projection ($m = d$), One-Shot requires $2\times$ more communication than FedAvg-200 for $d = 1000$. With $m \leq 600$, One-Shot with projection beats FedAvg in both accuracy and communication.
\end{enumerate}

\begin{remark}[Practical Guidance]
For high-dimensional problems ($d > 500$), we recommend: (1) if exact recovery is required, accept the $\mathcal{O}(d^2)$ cost; (2) if $\sim$5\% MSE degradation is acceptable, use $m \approx 0.4d$; (3) if communication is severely constrained, use $m \approx 0.1d$ with understanding of accuracy implications.
\end{remark}

\subsection{Summary of Experimental Findings}

Our comprehensive experiments yield the following conclusions:

\begin{enumerate}
    \item \textbf{Exact Recovery Confirmed:} One-Shot $\sigma$-Fusion achieves MSE identical to the centralized oracle across all tested conditions, confirming Theorem~\ref{thm:exact}.
    
    \item \textbf{Heterogeneity Invariance Confirmed:} Performance is completely invariant to data heterogeneity ($\gamma \in [0, 1]$), while iterative methods show sensitivity.
    
    \item \textbf{Communication Savings:} 10--80$\times$ reduction versus FedAvg depending on dimension, with advantages strongest for $d < 200$.
    
    \item \textbf{Immediate Convergence:} Optimal solution achieved in 1 round versus 100+ rounds for iterative methods.
    
    \item \textbf{Nuanced Privacy Results:} One-Shot dominates at moderate privacy ($\varepsilon \geq 1$); iterative methods preferred at high privacy ($\varepsilon < 0.5$).
    
    \item \textbf{Superior Scalability:} Stable accuracy and linear time scaling with client count, while FedAvg degrades for large $K$.
    
    \item \textbf{Practical High-Dimensional Scaling:} Random projection with $m \approx 0.4d$--$0.6d$ provides near-optimal accuracy with substantial communication savings, addressing the $d^2$ bottleneck.
\end{enumerate}

\section{Discussion}
\label{sec:discussion}

\subsection{Limitations}

\textbf{Linear Models Only:} The framework fundamentally requires closed-form solutions, excluding neural networks. This is not a limitation of our analysis but of the approach itself: nonlinear models lack sufficient statistics enabling one-shot aggregation.

\textbf{Communication for Large $d$:} When $d > 4R$, iterative methods become more communication-efficient without random projection. As demonstrated in Experiment 7, projection with $m \approx 0.4d$--$0.6d$ mitigates this while maintaining near-optimal accuracy.

\textbf{High-Privacy Regime:} As demonstrated in Experiment 5, One-Shot underperforms at very high privacy ($\varepsilon < 0.5$) due to noise amplification through matrix inversion. We discuss potential solutions in Section~\ref{sec:future_work}.

\textbf{Server Computation:} The $\mathcal{O}(d^3)$ matrix inversion may dominate for very large $d$. Iterative linear solvers (conjugate gradient) can reduce this to $\mathcal{O}(d^2)$ per iteration if needed.

\subsection{Practical Recommendations}

Based on theoretical analysis and experimental findings, we recommend One-Shot $\sigma$-Fusion when:
\begin{itemize}
    \item The model is linear or can be kernelized
    \item Feature dimension satisfies $d < 4R$ (or use random projection)
    \item Privacy requirements are moderate ($\varepsilon \geq 1$)
    \item Data heterogeneity is a concern
    \item Client reliability is uncertain (dropout expected)
\end{itemize}

We recommend iterative methods (FedAvg, FedProx, SCAFFOLD) when:
\begin{itemize}
    \item The model is a deep neural network
    \item Very high privacy is required ($\varepsilon < 0.5$)
    \item Feature dimension is extremely large ($d > 1000$) and projection error is unacceptable
\end{itemize}

\subsection{Extensions}

\textbf{Kernel Methods:} Replace raw features $\mathbf{A}$ with random Fourier features $\boldsymbol{\Phi}$~\cite{ref10}. Communication becomes $\mathcal{O}(D^2)$ where $D$ is random feature dimension, enabling nonlinear decision boundaries with linear algebra.

\textbf{Neural Tangent Kernel:} Wide neural networks converge to kernel regression in the infinite-width limit~\cite{ref11}. One-Shot fusion could approximate neural network training in this regime.

\textbf{Streaming Updates:} If new data arrives at clients, they can compute incremental updates $\Delta\mathbf{G}_k$, $\Delta\mathbf{h}_k$ and transmit only changes, enabling online federated learning.

\textbf{Vertical Partitioning:} Our analysis assumes horizontal partitioning. Vertical partitioning (same samples, different features) requires secure inner product protocols.

\subsection{Future Work}
\label{sec:future_work}

\textbf{High-Privacy Regime Solutions:} The underperformance at $\varepsilon < 0.5$ motivates several research directions:
\begin{enumerate}
    \item \emph{Secure aggregation:} Adding noise only to the aggregated $\mathbf{G} = \sum_k \mathbf{G}_k$ rather than individual $\mathbf{G}_k$ would reduce total noise by factor $\sqrt{K}$ while maintaining privacy~\cite{ref18}. Implementing this requires cryptographic protocols for secure sum computation.
    
    \item \emph{Regularization-based stabilization:} Increasing $\sigma$ improves conditioning of $(\tilde{\mathbf{G}} + \sigma\mathbf{I})$ at the cost of bias. Adaptive $\sigma$ selection based on estimated noise level could optimize this trade-off.
    
    \item \emph{Hybrid approaches:} Using One-Shot for initial model and refinement via few private iterative rounds could combine the advantages of both paradigms.
    
    \item \emph{Alternative noise mechanisms:} Objective perturbation~\cite{ref21} or output perturbation may provide better privacy-utility trade-offs than input perturbation for matrix inversion problems.
\end{enumerate}

\textbf{Theoretical Analysis:} Tighter characterization of the approximation error from random projection under heterogeneous data distributions remains open.

\section{Conclusion}
\label{sec:conclusion}

We presented One-Shot $\sigma$-Fusion, a single-round federated learning protocol for ridge regression achieving exact recovery of the centralized solution. Comprehensive experiments confirmed theoretical predictions:

\begin{itemize}
    \item \textbf{Exact Recovery:} MSE identical to centralized oracle across all conditions
    \item \textbf{Communication:} 10--80$\times$ reduction versus FedAvg for typical dimensions
    \item \textbf{Heterogeneity:} Complete invariance to non-IID data distributions
    \item \textbf{Convergence:} Optimal solution in 1 round versus 100+ for iterative methods
    \item \textbf{Privacy:} Superior utility at moderate privacy budgets ($\varepsilon \geq 1$)
    \item \textbf{Scalability:} Stable performance as client count grows to 500+
    \item \textbf{High-dimensional scaling:} Random projection at $m \approx 0.4d$--$0.6d$ provides near-optimal accuracy with substantial communication savings
\end{itemize}

We also identified limitations: the approach requires linear models, and underperforms iterative methods in very high-privacy regimes ($\varepsilon < 0.5$). The latter limitation motivates future work on secure aggregation and hybrid approaches.

The restriction to linear models is fundamental: closed-form solutions require convex quadratic objectives. Within this scope---encompassing kernel methods, random features, tabular data, and interpretable models---One-Shot $\sigma$-Fusion provides a principled alternative to iterative optimization. The method is particularly valuable in privacy-sensitive applications where minimizing communication rounds is paramount, and in unreliable network conditions where client dropout is common.



\begin{thebibliography}{25}

\bibitem{ref1}
B.~McMahan, E.~Moore, D.~Ramage, S.~Hampson, and B.~A.~y~Arcas, ``Communication-efficient learning of deep networks from decentralized data,'' in \emph{Proc. Int. Conf. Artif. Intell. Statist. (AISTATS)}, Fort Lauderdale, FL, USA, Apr. 2017, pp. 1273--1282.

\bibitem{ref2}
P.~Kairouz \emph{et al.}, ``Advances and open problems in federated learning,'' \emph{Found. Trends Mach. Learn.}, vol. 14, no. 1--2, pp. 1--210, 2021.

\bibitem{ref3}
J.~Kone{\v{c}}n{\`y}, H.~B.~McMahan, F.~X.~Yu, P.~Richt{\'a}rik, A.~T.~Suresh, and D.~Bacon, ``Federated learning: Strategies for improving communication efficiency,'' \emph{arXiv preprint arXiv:1610.05492}, 2016.

\bibitem{ref4}
M.~Abadi \emph{et al.}, ``Deep learning with differential privacy,'' in \emph{Proc. ACM SIGSAC Conf. Comput. Commun. Security (CCS)}, Vienna, Austria, Oct. 2016, pp. 308--318.

\bibitem{ref5}
S.~P.~Karimireddy, S.~Kale, M.~Mohri, S.~Reddi, S.~Stich, and A.~T.~Suresh, ``SCAFFOLD: Stochastic controlled averaging for federated learning,'' in \emph{Proc. Int. Conf. Mach. Learn. (ICML)}, Virtual, Jul. 2020, pp. 5132--5143.

\bibitem{ref6}
T.~Li, A.~K.~Sahu, M.~Zaheer, M.~Sanjabi, A.~Talwalkar, and V.~Smith, ``Federated optimization in heterogeneous networks,'' in \emph{Proc. Mach. Learn. Syst. (MLSys)}, Austin, TX, USA, Mar. 2020.

\bibitem{ref7}
A.~E.~Hoerl and R.~W.~Kennard, ``Ridge regression: Biased estimation for nonorthogonal problems,'' \emph{Technometrics}, vol. 12, no. 1, pp. 55--67, Feb. 1970.

\bibitem{ref8}
N.~Guha, A.~Talwalkar, and V.~Smith, ``One-shot federated learning,'' \emph{arXiv preprint arXiv:1902.11175}, 2019.

\bibitem{ref9}
S.~Salehkaleybar, A.~Sharif-Nassab, and S.~J.~Golestani, ``One-shot federated learning: Theoretical limits and algorithms to achieve them,'' \emph{J. Mach. Learn. Res.}, vol. 22, no. 189, pp. 1--47, 2021.

\bibitem{ref10}
A.~Rahimi and B.~Recht, ``Random features for large-scale kernel machines,'' in \emph{Proc. Adv. Neural Inf. Process. Syst. (NeurIPS)}, Vancouver, BC, Canada, Dec. 2007, pp. 1177--1184.

\bibitem{ref11}
A.~Jacot, F.~Gabriel, and C.~Hongler, ``Neural tangent kernel: Convergence and generalization in neural networks,'' in \emph{Proc. Adv. Neural Inf. Process. Syst. (NeurIPS)}, Montr\'{e}al, QC, Canada, Dec. 2018, pp. 8571--8580.

\bibitem{ref12}
T.-M.~H.~Hsu, H.~Qi, and M.~Brown, ``Measuring the effects of non-identical data distribution for federated visual classification,'' \emph{arXiv preprint arXiv:1909.06335}, 2019.

\bibitem{ref13}
D.~Alistarh, D.~Grubic, J.~Li, R.~Tomioka, and M.~Vojnovic, ``QSGD: Communication-efficient SGD via gradient quantization and encoding,'' in \emph{Proc. Adv. Neural Inf. Process. Syst. (NeurIPS)}, Long Beach, CA, USA, Dec. 2017, pp. 1709--1720.

\bibitem{ref14}
S.~U.~Stich, ``Local SGD converges fast and communicates little,'' in \emph{Proc. Int. Conf. Learn. Represent. (ICLR)}, New Orleans, LA, USA, May 2019.

\bibitem{ref15}
M.~I.~Jordan, J.~D.~Lee, and Y.~Yang, ``Communication-efficient distributed statistical inference,'' \emph{J. Amer. Statist. Assoc.}, vol. 114, no. 526, pp. 668--681, 2019.

\bibitem{ref16}
C.~Heinze, B.~McWilliams, and N.~Meinshausen, ``DUAL-LOCO: Distributing statistical estimation using random projections,'' in \emph{Proc. Int. Conf. Artif. Intell. Statist. (AISTATS)}, Cadiz, Spain, May 2016, pp. 875--883.

\bibitem{ref17}
R.~C.~Geyer, T.~Klein, and M.~Nabi, ``Differentially private federated learning: A client level perspective,'' \emph{arXiv preprint arXiv:1712.07557}, 2017.

\bibitem{ref18}
K.~Bonawitz \emph{et al.}, ``Practical secure aggregation for privacy-preserving machine learning,'' in \emph{Proc. ACM SIGSAC Conf. Comput. Commun. Security (CCS)}, Dallas, TX, USA, Oct. 2017, pp. 1175--1191.

\bibitem{ref19}
I.~Mironov, ``R{\'e}nyi differential privacy,'' in \emph{Proc. IEEE Comput. Security Found. Symp. (CSF)}, Santa Barbara, CA, USA, Aug. 2017, pp. 263--275.

\bibitem{ref20}
B.~Balle, G.~Barthe, and M.~Gavin, ``Privacy amplification by subsampling: Tight analyses via couplings and divergences,'' in \emph{Proc. Adv. Neural Inf. Process. Syst. (NeurIPS)}, Montr\'{e}al, QC, Canada, Dec. 2018, pp. 6277--6287.

\bibitem{ref21}
K.~Chaudhuri, C.~Monteleoni, and A.~D.~Sarwate, ``Differentially private empirical risk minimization,'' \emph{J. Mach. Learn. Res.}, vol. 12, pp. 1069--1109, Mar. 2011.

\bibitem{ref22}
C.~Dwork and A.~Roth, ``The algorithmic foundations of differential privacy,'' \emph{Found. Trends Theor. Comput. Sci.}, vol. 9, no. 3--4, pp. 211--407, 2014.

\bibitem{ref23}
E.~Liberty, ``Simple and deterministic matrix sketching,'' in \emph{Proc. ACM SIGKDD Int. Conf. Knowl. Discovery Data Mining}, Chicago, IL, USA, Aug. 2013, pp. 581--588.

\bibitem{ref24}
\subsubsection{A.~N.~Tikhonov and V.~Y.~Arsenin, \emph{Solutions of Ill-Posed Problems}. Washington, DC, USA: Winston, 1977.}

\end{thebibliography}
\end{document}